\documentclass{article}

\usepackage{microtype}
\usepackage{graphicx}
\usepackage{subfigure}
\usepackage{booktabs} 
\usepackage{hyperref}

\usepackage[accepted]{icml2024}

\usepackage{mathtools}
\usepackage{amsthm}
\usepackage{amsmath, amssymb}

\usepackage[capitalize,noabbrev]{cleveref}

\usepackage{amsfonts} 
\usepackage{array}
\usepackage{mathtools}
\usepackage{amssymb, amsthm, bm}

\usepackage{enumitem}

\usepackage{bbm}

\usepackage{siunitx}
\usepackage{wrapfig}

\theoremstyle{plain}
\newtheorem{theorem}{Theorem}[section]
\newtheorem{proposition}[theorem]{Proposition}
\newtheorem{lemma}[theorem]{Lemma}
\newtheorem{corollary}[theorem]{Corollary}
\theoremstyle{definition}

\theoremstyle{remark}

\usepackage[textsize=tiny]{todonotes}

\icmltitlerunning{A Theoretical Formulation of Many-body MPNN}

\begin{document}

\twocolumn[
    \icmltitle{A Theoretical Formulation of Many-body Message Passing Neural Networks}
    
    
    
    \icmlsetsymbol{equal}{*}
    
    \begin{icmlauthorlist}
    
    \icmlauthor{Jiatong Han}{oxford}
    
    \end{icmlauthorlist}
    \icmlcorrespondingauthor{Jiatong Han}{jiatong.han@cs.ox.ac.uk}
    
    \icmlaffiliation{oxford}{Department of Computer Science, University of Oxford}

    \icmlkeywords{graph neural networks, message-passing neural networks, graph spectral theory, many-body problem, graph convolutional networks}

    \vskip 0.3in
]



\printAffiliationsAndNotice{}  

\begin{abstract}
    We present many-body Message Passing Neural Network (MPNN) framework that models higher-order node interactions ($\ge 2$ nodes). We model higher-order terms as tree-shaped motifs, comprising a central node with its neighborhood, and apply localized spectral filters on motif Laplacian, weighted by global edge Ricci curvatures. We prove our formulation is invariant to neighbor node permutation, derive its sensitivity bound, and bound the range of learned graph potential. We run regression on graph energies to demonstrate that it scales well with deeper and wider network topology, and run classification on synthetic graph datasets with heterophily and show its consistently high Dirichlet energy growth. 

    We open-source our code at \url{https://github.com/JThh/Many-Body-MPNN}.

\end{abstract}

\section{Introduction} 
We study a generic graph setting where no information on distances between nodes or rotations of edges is available. We construct many-body message to increase the receptive field of a single message-passing step, than the two-body case, as explicitly encoding many-body interaction reduces the need of stacking layers for message-passing \citep{batatia2023mace}. 


Prior work such as ChebNet \citep{defferrard2016convolutional} approximates the multi-hop messages from neighbors through polynomial expansions to avoid directly computing the powers of adjacency matrices. However, the spectral filters of ChebNet are making signal transformations to graphs globally, despite its efforts to localize spectral filtering through Chebyshev expansion via finite support size $k$. We apply a simple yet effective localization by explicitly defining a series of motifs \citep{monti2018motifnet} where smaller-scale spectral filters are applied.

We modify motif Laplacian transform to distinguish the contribution from each node in the motif. This involves mapping the feature vectors to the Fourier basis of the concerned motif and scaling by the edge's curvature, before mapping them back to graph domain. We enumerate the neighboring sets of varied sizes and symmetrically aggregating the filtered outcomes to ensure permutation invariance.

Our contributions are to: 1) formulate many-body message-passing, prove its invariance to neighbor node permutation, and make explicit the contribution of every body-order interaction (\cref{fm:expanded}); 2) derive the new sensitivity and energy bound of the many-body interaction paradigm (\cref{fm:bound}, \cref{thm:energy}); 3) show many-body MPNN scale well with wider and deeper network topology and achieve significant energy growth on heterophilic graphs (Section \ref{exp:graph_reg} and \ref{exp:hetero}). 


\section{Preliminary}
We study general graph structure $G$ that have node set $\mathsf{N}$ and edge set $\mathsf{E}$ of finite sizes. We use $e_{i,j}$ to denote edge between node $i$ and $j$. We use $D$ for degree matrix, $A$ for adjacency matrix, $L = (D-A)$ for unnormalized graph Laplacian, $\mathcal{L} = I - D^{-\frac{1}{2}}AD^{-\frac{1}{2}}$ as symmetric normalized Laplacian. We define correlation order $\nu$ to be the order to which we model the body interactions explicitly in our framework, including the central node. 


For different correlation orders, we use motifs (undirected graph sub-structures) to represent the local graph formed by the central node and its neighbors in consideration. We define $k$-motifs to specifically refer to undirected tree graphs with one central node and $(k-1)$ branches to localize many-body interactions of different orders. The adjacency matrix (and hence Laplacian $L$) of motifs can be assigned weights such as edge curvatures for localization. We rely on Balanced Forman curvature, a combinatorial definition of Ricci curvature as established in \citet[Def.~1]{topping2022understanding}, and we hereafter refer to it as Ricci curvature and the function $\text{Ricci}(e_{i,j})$ for curvature of edge $e_{i,j}$.


\paragraph{Potential function.} A graph's potential $\mathcal{E}$ can be measured with Dirichlet energy. Dirichlet energy is defined as 


\begin{equation}\label{fm:dirich}
    Tr(\bm{X}^T \mathcal{L}  \bm{X}) = \sum_{(i,\ j)\in \mathsf{E}}{\left\|\frac{\bm{X}_i}{\sqrt{d_i}} -\frac{\bm{X}_j}{\sqrt{d_j}} \right\|_2^2},
\end{equation} 
where $\mathcal{L} $ is the normalized Laplacian, $\bm{X}$ is the set of node features, and $d$'s are node degrees.

\section{Formulation of Many-body MPNN}
The generalized message construction involving all correlation orders $\nu$\footnote{Here we have overloaded the notation $\nu$ to denote many-body interaction function, aside from its initial meaning of correlation orders.} as defined by \citet{batatia2023mace} is :

\begin{equation}\label{fm:vlmpnn}
    m_i^{(t)} = \sum_{j} \nu_1(\sigma_i^{(t)}, \sigma_j^{(t)}) + \sum_{j_1,j_2} \nu_2(\sigma_i^{(t)}, \sigma_{j_1}^{(t)}, \sigma_{j_2}^{(t)}) + \cdots,
\end{equation}

where $j$'s are node $i$'s neighbors and $\sigma_j^{(t)}$ is the learned representation of node $j$ for computing $\nu$-th correlation strength at layer $t$.

For modelling interactions of different body orders, we follow the format of \cref{fm:vlmpnn} and avoid explicit message passing via spectral filters. We start by computing the two-body interaction messages using graph convolution. This involves apply graph Fourier transform to neighboring nodes, weighting their contributions via learnable parameters, and localizing the signal transformation through second-order Chebyshev expansion. 

The two-body interaction component of our model captures the pairwise relationships between nodes. By incorporating node features from previous iterations, we compute the interaction messages as follows:

\begin{equation}\label{fm:2body}
    \bm{X}^{(t)} = \bm{U}^\top g_{\theta_2}(\bm{\Lambda}) \bm{U} \bm{H}^{(t-1)},
\end{equation}
where \( \bm{X}^{(t)} \) represents the two-body interaction matrix at iteration \( t \), \( \bm{U} \) is the matrix of eigenvectors of $G$'s normalized Laplacian. The function \( g_{\theta_2}(\bm{\Lambda}) \) denotes a second-order Chebyshev polynomial expansion applied to the diagonal matrix \( \bm{\Lambda} \), which holds the eigenvalues of the Laplacian matrix of the graph \( G \). 

Higher-order interaction message (\cref{fm:msg}) captures the complex interactions involving more than two nodes, utilizing motif-based structures for higher-order relations. We draw on the intuition from Windowed Fourier Transform (WFT) to apply graph signal filtering to subgraphs, specifically motifs sized according to the correlation order. These motifs are standardized undirected subgraphs centered around a node with neighborhoods of size ($\nu$). The output message for each node maintains permutation invariance by explicitly enumerating (by $\eta_{\nu}$) all neighboring sets of size $k \le \nu$. 

To better locate each motif in global graph $G$, we modify the motif Laplacian to account for Ricci curvature. Edges with more negative Ricci curvature are assigned higher positive weights in the local Laplacian of $G$'s motif, reflecting more critical connections in the graph structure: 


\[
L_{\text{Ricci-J}_{(i,j)}} = 
\begin{cases}
    \sum_{j \neq i} \text{Ricci}(e_{i, j}) & \text{if } i = j \\
    - \text{Ricci}(e_{i, j}) & \text{if } i \neq j \\
    0 & \text{otherwise},
\end{cases}
\]

where $J$ is the neighbor node set of node $i$ and $j \in J$. 



The higher-order message $\bm{Y}^{(t)}$ is hence formulated as (with $\nu \ge 3$):

\begin{equation}
\bm{Y}^{(t)}_{i} = \prod_{k=3}^{\nu} \sum_{\substack{J \subseteq \eta(N(i)) \cup \{i\}, \\ |J| = k}} \bm{U}_k^\top g_{\theta_{k}}(\bm{\Lambda}_{\text{Ricci-J}}) \bm{U}_k \bm{H}^{(t-1)}_{J},\label{fm:msg}
\end{equation}

where $\eta$ is the enumeration of element sets from the neighbors of node $i$. $U_{k}^\top$ is the inverse or transpose of eigenvectors of $L_{\text{Ricci-J}}$. And $g_{\theta_{k}}(\bm{\Lambda}_{\text{Ricci-J}})$ term is defined as:

\begin{equation}\label{fm:cheb}
g_{\theta_{k}}(\bm{\Lambda}_{\text{Ricci-J}}) = \sum_{k'=1}^{k} \theta_{k,k'} T_{k'}(\widetilde{\bm{\Lambda}_{\text{Ricci-J}}}),
\end{equation}

where $\widetilde{\bm{\Lambda}_{\text{Ricci-J}}} = \frac{2\bm{\Lambda}_{\text{Ricci-J}}}{\lambda_{\text{max-J}}} - I_{\lvert J \rvert}$ and $\bm{\Lambda}_{\text{Ricci-J}}$ is the diagonal matrix of eigenvalues of $\lvert J \rvert$-motif's Laplacian $L_{\text{Ricci-J}}$ whose weights are obtained from the edge set $\mathsf{E}_{i, J} \coloneqq \{ e_{i, j} \mid j \in J \}$ with respect to $G$, with $\lambda_{\text{max-J}}$ representing the largest eigenvalue of the Laplacian $L_{\text{Ricci-J}}$.





And the feature update equation, which updates the node features by combining the self-features and aggregated neighbor messages, establishes a residual connection:
\begin{align}
    \bm{h}_i^{(t)} &= \bm{h}_i^{(t-1)} + \bm{W}^{(t)} m^{(t)}_i \nonumber \\
    &=  \bm{h}_i^{(t-1)} + \bm{W}_{x}^{(t)}\bm{X}^{(t)}_{i} + \bm{W}_{y}^{(t)}\bm{Y}_i^{(t)} \label{fm:update}
\end{align}


\section{Many-Body Mixing Bound: Sensitivity Bound}\label{sec:sensit}

For node $i$ at layer $t$, the update function of all body orders can be formulated as:

\begin{align}
    \bm{h}_i^{(t)} &= \sum_{k'=0}^{2} \theta_{2,k'}^{(t)} T_{k'}(\widetilde{\mathcal{L}}) \bm{h}_i^{(t-1)} \nonumber \\
    &\quad +  \prod_{k=3}^{\nu} \sum_{\substack{J \subseteq \eta_k(N(i))}} \sum_{k'=1}^{k} \theta_{k,k'}^{(t)} T_{k'}(\widetilde{L_{\text{Ricci-J}}}) \bm{H}_J^{(t-1)} 
    \label{fm:expanded}
\end{align}

, where $\bm{h}_i^{(0)} = \bm{x}_i$, with the residual term included when $k' = 0$.

The over-squashing effect can be understood with node's representation $\bm{h}_u^{(r)}$ failing to be affected by some input feature $\bm{x}_{v}$ of node $v$ at distance $r$ from node $u$. We use the Jacobian ($\partial \bm{h}_{u}^{(r+1)} / \partial \bm{x}_{v}$) to assess over-squashing and derive sensitivity bounds, similar as \citet{topping2022understanding}.

\begin{theorem}\label{thm:bound}
    The sensitivity bound of a many-body MPNN with update function as defined in \cref{fm:expanded} is given by

    \begin{equation}
        O\left(\left| \frac{\partial \bm{h}_{u}^{(r+1)}}{\partial \bm{x}_{v}} \right|\right) =  O\left((\mathcal{A}^{r})_{vu} \bm{x}_v^{\nu^{r}}\right) \label{fm:bound}
    \end{equation}
    where $\nu \ge 2$ is the correlation order, $r$ is the shortest distance between node $v$ and node $u$.
    
\end{theorem}
The powers of $x_v$ in \cref{fm:bound} come from the product term $\prod$ from \cref{fm:msg}, different from an MPNN's sensitivity bound that does not usually involve node initializations \citep{digiovanni2024does}, as they are negligible constants. 



\section{Properties of Many-Body MPNN}\label{sec:prop}

\subsection{Invariance Property}

We reconstruct message construction as $\mathcal{F}_m(\bm{H}^{(t-1)}, L) = W^{(t)} m^{(t)}$ from \cref{fm:update}. 


\begin{theorem}\label{thm:invar}
The message construction function ($\mathcal{F}_m$) is invariant to the permutation of neighbors $N(i)$ of input nodes $i \in \mathsf{N}$, considering different body numbers $\nu$, i.e.
\begin{equation}
    \mathcal{F}_m(\bm{H}^{(t-1)}, L) = \mathcal{F}_m(\bm{H}^{(t-1)}, \pi^\top L\pi),
\end{equation}
assuming $\pi$ is the permutation matrix.
\end{theorem}


\subsection{Energy Bounds}

\begin{theorem}
   Many-body MPNN with finite number of layers ($t$) learns the graph potential ($\mathcal{E}$) into a bounded range: $$\bigg[0,\quad \lambda_{\text{max}} \lvert \mathsf{N} \rvert\bigg(\prod_{\nu} \nu^t{d_{\text{max}} \choose \nu-1}^t \prod_t \bm{w}^{(t)}_{\nu} \bm{h}\bigg)^2  \bigg],$$ where $d_{\text{max}}$ is the maximum node degree and correlation order $\nu \ge 2$, and $\bm{w}^{(t)}, \bm{h}$ are assumed constant upper bounds for weights of layer $t$, and the initial node features. \label{thm:energy}
\end{theorem}

We derive from \cref{thm:energy} that many-body MPNN has higher energy upper-bound than ChebNets, and higher-order terms produce strictly more energy than lower-order terms, given the same number of layers $t$. See \cref{app:energy}.



\subsection{Complexity Bounds}

\begin{proposition}\label{thm:time}
    There exists a simplified implementation of many-body MPNN whose runtime is linear in $(\lvert \mathsf{E} \rvert+\lvert \mathsf{N} \rvert)$ and has time complexity $O(2\lvert \mathsf{E} \rvert + \lvert \mathsf{N} \rvert(d_{\text{max}})^{\nu-1})$ per layer, where $d_{\text{max}}$ is the maximum node degree and correlation order $\nu \ge 2$. 
\end{proposition}


We prove in \cref{thm:time} that despite the additional complexity of constructing higher-order messages, the many-body implementation can execute at comparable speed as the two-body case in its optimal state. See \cref{app:benchmark}.

\section{Experiments}
We demonstrate the proven properties of many-body MPNN through graph regression and node classification tasks, illustrating its capability to scale with more extensive and deeper architectures and to capture complex local node interactions.

\subsection{Experimental Settings.}\label{sec:exp}
Due to the sheer computational complexity of approximating Ricci curvatures and weighted motif Laplacian, we follow \cref{thm:time} to simplify our implementation. We compare our model performances with two-body MPNNs and other graph convolutional networks such as GCNs \citep{kipf2017semisupervised} and ChebNets \citep{defferrard2016convolutional}.




\subsection{Regressing Synthetic Random Graph Energies}\label{exp:graph_reg}

We experiment with random graphs and synthesized energy functions. We design our energy functions through applying non-linear transformations such as exponential or logarithmic functions on graph average shortest path lengths or average clustering coefficient, to emphasize either distant node mixing or local node clustering. 

We experiment with 100 Erdős–Rényi Graph having 500 to 700 nodes with edge probabilities 0.15 to 0.3. We observe in \cref{fig:random} that when emphasizing node distances, stacking more many-body MPNN layers or decreasing hidden dimensions let it perform observably better than others, yet \textbf{the opposite} is witnessed when emphasizing local clustering, which we understand as wider neural networks benefit learning more complex local node interactions. 


\begin{figure*}[ht]
    \centering
    \includegraphics[width=0.85\linewidth]{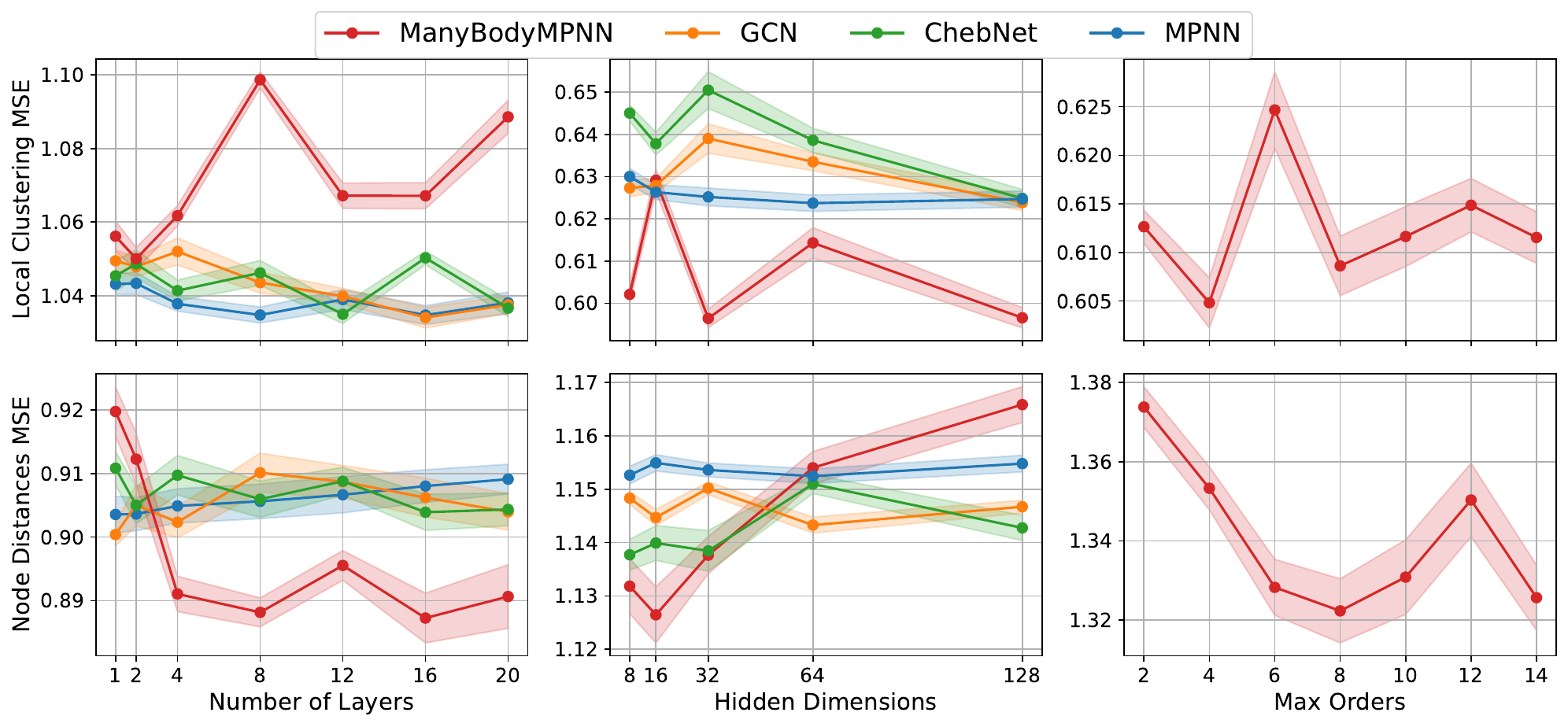}
    \caption{Test MSE Losses on 100 synthetic random graphs. We show the impact of varying number of layers, hidden dimensions, and max correlation orders on the energy regression. The upper row is regressing graph energies emphasizing clustering, with the lower emphasizing node distances.}
    \label{fig:random}
\end{figure*}

\subsection{Classifying Synthetic Heterophilic Graph Nodes}\label{exp:hetero}

We synthesize graph datasets for node classification task that exhibits highly heterophilic properties to illustrate graph learning capabilities. We generate a single heterophilic graph with 10,000 nodes with 7 classes, each of which has 1,433 dimensional features, with an average node degree of 10, as a larger replacement of Cora \citep{chen2018fastgcn}. We measure the test accuracy and the Dirichlet energy growth over epochs (see Figure \ref{fig:hetero_test}). 

\begin{figure}[ht]
    \centering
    \includegraphics[width=\linewidth]{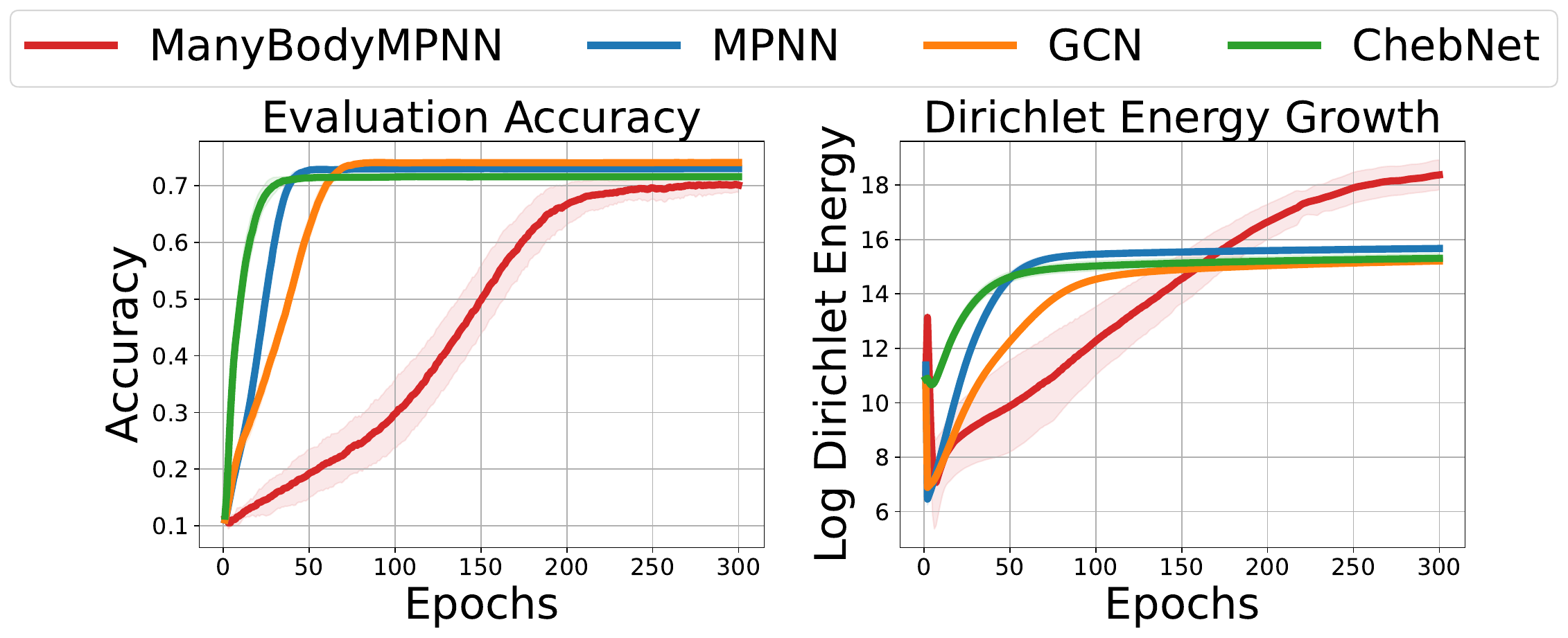}
    \caption{Test accuracy and Dirichlet energy growth of models over 300 epochs on a synthetic heterophilic graph.}
    \label{fig:hetero_test}
\end{figure}

Many-body MPNN generates much higher energy than other networks, despite its slower convergence, signalling its capability of learning the representations of contrasting nodes quite differently.

\subsection{Efficiency Benchmarks}\label{app:benchmark}
We benchmark model speeds for different layer counts in \cref{fig:bench}. We observe that with runtime increasing linearly with layer counts, many-body MPNN runs 3.44 times slower than ChebNet when having 20 layers, mainly attributed to higher-order interactions. We may further improve the running time as proposed in \cref{thm:time} in future efforts.

\begin{figure}[ht]
    \centering
    \includegraphics[width=0.85\linewidth]{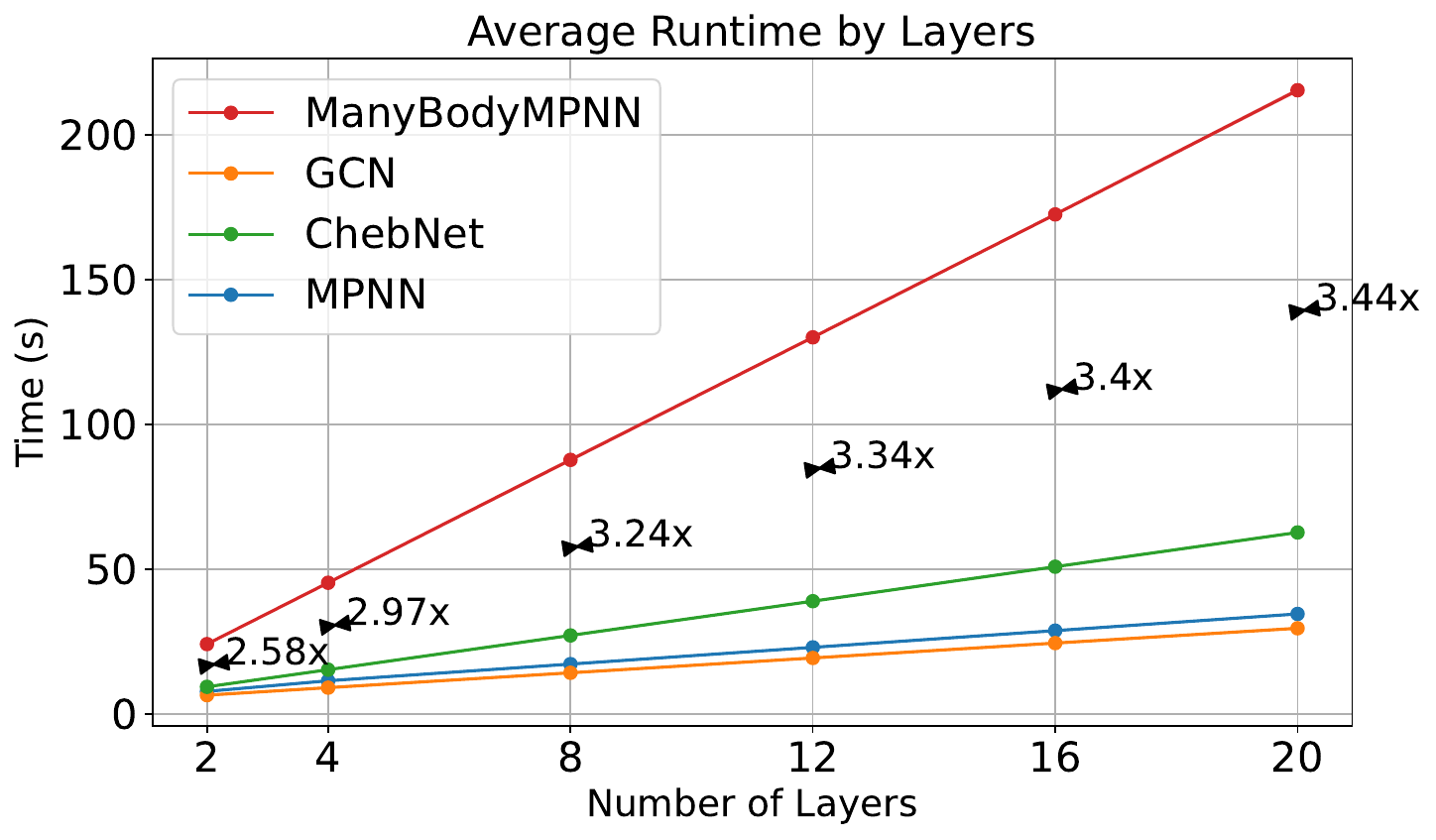}
    \caption{Runtime for different models, averaged over 30 runs. The models have varying numbers of layers, 16 hidden dimensions, and 4 correlation orders. The support for the Chebyshev expansion is 4. The batch size is 4, and the number of epochs is 50. Benchmarked on one NVIDIA RTX 2080 Ti GPU.}
    \label{fig:bench}
\end{figure}  

\section{Discussion}\label{sec:discuss}


\paragraph{The effectiveness of many-body MPNN for downstream tasks.} Our preliminary results show that many-body formulation scales well with increasing network depths when modeling the distant node interactions, and with increasing network widths when modeling complex local node interactions. It captures both aspects which are not easily learned through stacking more MPNN layers due to over-squashing \citep{digiovanni2024does}, or introducing wider layers.

\paragraph{Localized spectral filtering and efficiency concerns.} ChebNet applies signal processing to the entire graph, while our approach is more localized by applying to motifs while accounting for the edge contributions through its curvatures. We make the contribution of every sized motif explicit, and capture more subtle signals that only exists locally. One direct benefit of our approach is the higher energy upper-bound (\cref{thm:bound}, \cref{thm:chebnet}); on learning heterophilic graph node labels, it generates more energy over learning for contrasting node labels than ChebNet through its many-body interactions, despite its slightly harder convergence with its heavy parameterization, and we leave its improvement for future work. 



\paragraph{Towards learnable embeddings.} For input graphs with edge features, the Ricci curvatures can actually be replaced with learnable parameters \citep{batatia2023mace}. Additionally, depending on what other information is available in the graphs, such as node positional encoding, distances or commute time, etc. \citep{Black2023}, we may replace Ricci curvatures with these metrics to make their impact more explicit, which might open up exciting avenues for many-body MPNN research. 

\section{Related Work}



\paragraph{MACE.} The MACE (Message passing neural network for Atom-Centered Potentials) framework \citep{batatia2023mace}, symbolized by the correlation order $\nu$, predicts atomistic potentials within molecular structures. It adapts to the permutation invariance inherent in molecules, due to the indistinguishability of atoms of the same element, and ensures energy conservation. MACE can be expressed by:
\begin{equation}
    E(\mathbf{r}^N) = \sum_{i=1}^N \nu_i (\mathbf{r}_i, \mathbf{r}_{i}^{'}, \ldots, \mathbf{r}_{i}^{(\nu)})
\end{equation}
where $E(\mathbf{r}^N)$ represents the potential energy of a system with $N$ atoms, $\mathbf{r}_i$ denotes the position of the $i$-th atom, and $\mathbf{r}_{i}^{(\nu)}$ denotes its $\nu$-th order interaction.

\paragraph{Graph Convolutional Neural Networks.} ChebNet \citep{defferrard2016convolutional} introduces localized spectral filters within graph neural networks. The spectral filters are based on Chebyshev polynomials, $T_k(x)$, which serve as an efficient approximation to the graph Laplacian's eigendecomposition, facilitating faster convolutions. For a graph signal $\mathbf{x}$ and a filter $g_\theta$, the operation in ChebNet can be formulated as:
\begin{equation}
    \mathbf{x} * g_\theta = \sum_{k=0}^{K-1} \theta_k T_k(\widetilde{\mathbf{L}}) \mathbf{x}
\end{equation}
where $*$ denotes the convolution operation, $\theta$ is a vector of Chebyshev coefficients, $\widetilde{\mathbf{L}}$ is the scaled Laplacian, and $K$ represents the order of the polynomial approximation. 

\paragraph{Understanding Over-Squashing via Graph Curvature \citep{digiovanni2023oversquashing}.} We adapt from the Balanced Forman curvature formulation from \citet{digiovanni2023oversquashing} to weigh the edge importance in higher-order message construction. We derive the sensitivity bound (Theorem \ref{thm:bound}) based on the bound for two-body MPNN from \citet[Lemma 1]{digiovanni2023oversquashing}, and demonstrate the many-body formulation is capable to scale to deeper network topology.

\paragraph{Understanding Graph Convolutions via Energies \citep{rusch2023survey, topping2022understanding, Chamberlain2021}.} We leverage the concept of Dirichlet energy and measure how well models embed node differently enough in graphs with heterophily. While we do not ablate model layers while doing node classification, we see many-body MPNN is less likely to hit the energy flow bottleneck than other models through its high energy generation.


\section{Conclusion}


In this paper, we develop a theoretical formulation of many-body MPNN, which models higher-order node interactions beyond traditional two-body interactions. We address ChebNet's limitations in localizing graph transform effects by designing a technique that enumerates tree-shaped motifs for each correlation order and applies Chebyshev-expanded spectral filters. We derive sensitivity and energy bounds for many-body MPNN and evaluate its performance on synthetic graph energy regression and heterophilic node classification tasks. Our results show that many-body MPNN scales effectively with network depth and width, produces high energy, and maintains test accuracy on par with other convolutional networks, while exhibiting reduced susceptibility to either over-squashing or over-smoothing.

\section*{Acknowledgement}

We acknowledge the contribution of ideas from Francesco Di Giovanni (\href{mailto:francesco.di.giovanni@valencelabs.com}{francesco.di.giovanni@valencelabs.com}) and Michael Bronstein (\href{mailto:michael.bronstein@cs.ox.ac.uk}{michael.bronstein@cs.ox.ac.uk}) during their time at the University of Oxford. The paper was extended from a course (\href{https://www.cs.ox.ac.uk/teaching/courses/2023-2024/geodl/}{Geometric Deep Learning}) assessment paper solely written by the author.






\bibliographystyle{icml2024}
\bibliography{sample}
\clearpage
\newpage

\appendix


\newpage
\onecolumn

\section{Proofs}

\subsection{Discussion of \cref{thm:bound}}
The hidden feature $\bm{h}_i^{(t)}$, computed by a many-body MPNN with $t$ layers as in \cref{fm:expanded} is a differentiable function of the input node features $\{\bm{x}_1, \ldots, \bm{x}_n\}$.


We simplify \cref{fm:expanded} into more plain expressions in \cref{fm:plain}. We denote parameter tensors multiplied with constants (including edge curvatures) as $\bm{c}^{(a)}_{b,d}$'s for brevity.

\begin{equation}\label{fm:plain}
    \bm{h}_i^{(t)} = \bm{c}_0^{(t)}  \bm{h}_i^{(t-1)} \nonumber + \prod_{k=2}^{\nu} \Bigg( \bm{c}_{k,1}^{(t)} \bm{h}^{(t-1)}_{o(1)} + \bm{c}_{k,2}^{(t)} \bm{h}^{(t-1)}_{o(2)} + \cdots + \bm{c}_{k,\lvert N(i) \rvert}^{(t)} \bm{h}^{(t-1)}_{o(\lvert N(i) \rvert)} \Bigg) 
\end{equation}

, where $o(j)$ is used to denote the identity of $j$-th neighbor of node $i$.


With \cref{fm:plain}, we may relate the shortest path $\bar{p}$ to their sensitivity bounds. Assume $\mathcal{P}(v, u)$ is the set of paths connecting nodes $v$ and $u$, and $\ell(p)$ as length of path $p \in \mathcal{P}$, with $\ell(\bar{p}) = r$. Let $p_i$ be the $i$-th node along path $p$ and $\bm{S}^{(r)}_{p_i}= (\sum_{j \in N(p_i)} \bm{h}^{(r)}_j)$ be the direct sum of neighborhood of node $p_i$ at layer $r$. We may derive $\bm{h}^{(r+1)}_{u}$ by induction. The first term $\bm{h}^{(1)}_{p_1}$ along the path that contains $x_v$ terms can be derived as:

\begin{equation}\label{eq:induct1}
    \bm{h}^{(1)}_{p_1} = \sum^{\nu-1}_{k=1} \bm{c}_{p_1,k}\cdot (\bm{x}_v)^{k}\cdot (\bm{S}^{(0)}_{p_1})^{(\nu-k)}
\end{equation}

And $\bm{h}^{(r+1)}_u$ can be inductively derived as:

\begin{equation}\label{eq:induct2}
    \bm{h}^{(r+1)}_u = \sum_{p \in \mathcal{P}(v, u)}\sum^{\nu-1}_{k=1} \bm{c}_{p_{r},k}\cdot (\bm{h}^{(r)}_{p_{r}})^{k} \cdot (\bm{S}^{(r)}_{p_r})^{(\nu-k)}\\
\end{equation}

By induction, and the fact that $\bar{p} \in \mathcal{P}(v, u)$ can be obtained from powers of symmetrically normalized adjacency matrix $\mathcal{A}$ and the gradients from other paths would not have reached $u$ after $r$ propagations, we may finally derive the bound of $\lvert \partial \bm{h}_{u}^{(r+1)} / \partial \bm{x}_{v} \rvert$ as in \cref{thm:bound}.

\subsection{Proof of \cref{thm:invar}}


\begin{proof}
We assume that node features are initialized uniformly across graphs. We proceed by inducting on $\nu \ge 2$ (assuming that there are at least 2 bodies involved in message construction). 

When $\nu = 2$, $m^{(t)} = X^{(t)}$. From \cref{fm:2body}, we may have the permutation matrix moved inside the Laplacian's eigen-decomposition:
\begin{equation}
    \mathcal{F}_m(\bm{H}^{(t-1)}, \pi^\top L \pi) 
    = W^{(t)} \bm{U}^\top \left( \pi^\top g_{\theta_2}(\bm{\Lambda}) \pi \right) \bm{U} \bm{H}^{(t-1)} 
\end{equation}

Since $\Lambda$ is a diagonal matrix and $g_{\theta_2}(\bm{\Lambda})$ is linearly parameterizing the Chebyshev expansion of $\Lambda$, we may confirm that $\pi^\top g_{\theta_2}(\bm{\Lambda})\pi$ is an actual permutation of original outcomes. Due to the assumption that node features are initialized the same, final weighted outcome by $W^{(t)}$ will also be the same, and \textbf{by induction}, $\mathcal{F}_m(\bm{H}^{(t-1)}, L) = \mathcal{F}_m(\bm{H}^{(t-1)}, \pi^\top L\pi)$, which preserves permutation invariance.

For $\nu \ge 3$, the message construction includes the higher order term $Y^{(t)}$ from \cref{fm:msg}. Due to the fact that permutation of node identities will not alter the graph topology, for each enumerated neighbor set $J \in \eta(N(i))$, the edge set $\mathsf{E}_{i, J}$'s Ricci curvatures will not change. Hence, the weighted motif's Laplacian $\mathcal{L_{\text{Ricci-J}}}$ and its eigenvalues will be unaltered. With a fair $\eta$, each neighbor node should have equal probabilities to hold any position (in the local motif's topology) of any motif sizes. Hence permutation of node identities will not affect the message construction.
\end{proof}

\subsection{Proof of \cref{thm:energy}}\label{app:energy}
\begin{proof}
    We first prove the bound for a single-layer case, and then generalize the result to any number of layers. We assume that the node features at layer 0 for node $i \in \mathsf{N}$ are initialized with $X_i = \lvert \bm{H}_i^{(0)} \rvert \le \bm{h}$. We assume linear weights are initialized with $\lvert \theta_k \rvert \le \bm{w}_k$.

    We analyze the two-body $\bm{X}$ and many-body interaction term $\bm{Y}$ for node features $\bm{H}^{(1)}$ to approximate the bounds of $\mathcal{E}$ of graph after applying \cref{fm:update}'s updates. 

    For two-body interaction term $\bm{X}$, we leverage the fact that graph Laplacian $\mathcal{L}$ is symmetric and positive semi-definite, and hence its eigenvalues are bounded by $[0, \lambda_{\text{max}}]$. After shifting its eigenvalues to be $\Tilde{\bm{\Lambda}}$ as in \cref{fm:cheb}, and since Chebyshev polynomials applied on $[-1,1]$ are bounded by $[-1,1]$, it is straightforward that
    \begin{equation}
        \bm{X}^{(1)}_i \in [-d_{\text{max}}\bm{w}_2\bm{h}, d_{\text{max}}\bm{w}_2\bm{h}], \qquad \text{for } i \in \mathsf{N}
    \end{equation}

    For many-body interaction term $\bm{Y}$, there are finite number of neighboring sets for each node $i$ to form unique motifs, and with similar analysis, 
    \begin{equation}
        \lvert \bm{Y}^{(1)}_i \rvert \le \prod^{\nu}_{3} \nu{d_{\text{max}} \choose \nu-1}\bm{w}_{\nu}\bm{h}, \qquad \text{for } \nu \ge 3\text{, } i \in \mathsf{N}.
    \end{equation}

    Hence, $H^{(1)}_i \in [-\prod_{\nu} \nu{d_{\text{max}} \choose \nu-1}\bm{w}_{\nu}\bm{h}, \prod_{\nu} \nu{d_{\text{max}} \choose \nu-1}\bm{w}_{\nu}\bm{h}$], after we combine the two interaction terms.

    Then the graph potential is computed and bounded as 
    \begin{equation}
        \mathcal{E} \coloneqq H^{(1)T} \mathcal{L} H^{(1)} \in  \bigg[0, \lambda_{\text{max}} \lvert \mathsf{N} \rvert\bigg(\prod_{\nu} \nu{d_{\text{max}} \choose \nu-1}\bm{w}_{\nu}\bm{h}\bigg)^2  \bigg],
    \end{equation}
    where the maximum is obtained when $H^{(1)}$ aligns with the eigenvector corresponding to $\lambda_{\text{max}}$ and is at its maximal magnitude, and the minimum is zero, when graph nodes have the same embeddings.

    We generalize the result to $t \ge 2$ layers. It can be shown by simple induction that $\lvert H^{(t)} \rvert \le \prod_{\nu} \nu^t{d_{\text{max}} \choose \nu-1}^t \prod_t \bm{w}^{(t)}_{\nu} \bm{h}$.


    And the energy bound is hence $\bigg[0, \lambda_{\text{max}} \lvert \mathsf{N} \rvert\bigg(\prod_{\nu} \nu^t{d_{\text{max}} \choose \nu-1}^t \prod_t \bm{w}^{(t)}_{\nu} \bm{h}\bigg)^2  \bigg]$.
\end{proof}

\begin{lemma}\label{thm:higher}
    Higher-order interaction term generates more energy than lower-order terms.
\end{lemma}

\begin{proof}
    With the same number of layers $t$, node and linear weight bounds ($\bm{h}$ and $\bm{w}$), the remaining higher order term $\nu (d_{\text{max}})^{\nu-1}$ grows monotonically with $\nu$.
\end{proof}

\begin{corollary}\label{thm:chebnet}
    ChebNet has strictly lower energy upper-bound than many-body MPNN, with a single layer.
\end{corollary}
This is a corollary to Theorem \ref{thm:energy} since many-body formulation is identical to ChebNet when $\nu=2$, and the maximum energy contributed by many-body interaction term $Y$ is strictly positive (and higher than ChebNet by \ref{thm:higher}).

\subsection{Proof of \cref{thm:time}}
We will prove the time complexity upper-bound to be $O(\lvert \mathsf{E} \rvert d_{\text{max}}^2 + 2\lvert \mathsf{E} \rvert + \lvert \mathsf{N} \rvert(3d_{\text{max}})^{\nu-1}\nu^3)$, and discuss the result in paper on a separate note, following this proof.

\begin{proof}    
    There are three major parts of computation within a many-body MPNN update function: Balanced Forman curvatures for $G$, two-body interaction term, and many-body interaction term.
  
    Computing the Balanced Forman curvature on $G$ takes $O(\lvert \mathsf{E} \rvert d_{\text{max}}^2)$, according to the formulation in \citet{topping2022understanding}. The curvature value for each edge in $\mathsf{E}$ is pre-computed for constructing each motif's weighted Laplacian matrix. 

    For two-body interaction term, it has the same time complexity as ChebNet \citep{defferrard2016convolutional}, which is $O(\lvert \mathsf{E} \rvert \cdot K)$ where $K=2$ is the expansion order of Chebyshev polynomials from \cref{fm:2body}. 
    

    For many-body interaction term, we make a simplification that Balanced Forman curvatures (ranging from $-2$ to positive infinity) are rounded into $\{-1,0,1\}$ by the sign of continuous curvature values. The number of uniquely weighted motif's Laplacian of correlation order $\nu$ is:

    \begin{equation}
        \#{\mathcal{L}_\text{motif-J}} \coloneqq \sum_{i \in \mathsf{N}} {d_i \choose \nu-1} 3^{\nu-1}
    \end{equation}

    Since eigen-decomposition of Laplacians of shape $(\nu, \nu)$ takes $O(\nu^3)$ time, \cref{fm:msg} takes $O(\lvert \mathsf{N} \rvert(3d_{\text{max}})^{\nu-1}\nu^3 )$. And the overall time complexity of simplified many-body MPNN per layer is $O(\lvert \mathsf{E} \rvert d_{\text{max}}^2 + 2\lvert \mathsf{E} \rvert + \lvert \mathsf{N} \rvert(3d_{\text{max}})^{\nu-1}\nu^3)$.

    

\end{proof}

On a side note, we may further simplify many-body MPNN formulation through pre-computing the eigen-decomposition of motif's unweighted Laplacian, and instead of assigning edge curvatures as Laplacian weights, we learn the motif edge contribution through learnable parameters specific to each correlation order, from the Chebyshev expansion process (\cref{fm:cheb}). Since there is no need to compute the curvature values, the time complexity of this further simplified many-body MPNN is $O(2\lvert \mathsf{E} \rvert + \lvert \mathsf{N} \rvert(d_{\text{max}})^{\nu-1})$.

\section{Experiments}

\subsection{Details on Experiment Configurations}
For heterophilic graph node classification, We use 4 convolutional layers; we set correlation order $\nu$ to be 5 for many-body MPNN and Chebyshev filter size to be 3 for ChebNet. For synthetic spine graph mixing power experiments (\cref{app:mix}), we train many-body MPNN with layer number in $[4,8,13,25]$, widths in $[8,32,128,256]$, $\nu=5$ and Chebyshev filter size to be 3, and train 100 epochs with initial learning rate of 0.01 via Adam optimizer \citep{Kingma2015}. All experiments are bootstrapped with 10 runs using random seeds.

For graph regression tasks, we generally vary the graph types, model depths, and max orders, as shown in \cref{fig:random}, \ref{fig:layer}, and \ref{fig:maxorder}.

\subsection{Details on Synthesized Heterophilic Graphs and Missing Experimental Details}\label{app:mix}

We generate graphs with heterophily that often have contrasting node labels between neighbors. We synthesize such graphs by firstly generating random node labels according to the number of classes, then making edges between nodes with different labels at a higher probability $p_{\text{hetero}}$ (e.g. 0.8). 


We learn from Theorem \ref{thm:energy} that many-body MPNN may generate higher energies than ChebNet, which makes it less prone from over-smoothing issues. It is notable that while generating much energy, it does not suffer from over-fitting. 

To isolate the contribution of each interaction term of many-body MPNN, we plot the logarithmic energy growth for each $k$-body component, averaged by layers (\cref{fig:indv_growth}). We observe that the energy contribution is increasing over training process, while higher-body term contributes consistently higher energies as stated in \cref{thm:higher}.

\begin{figure}[ht]
    \centering
    \includegraphics[width=0.7\linewidth]{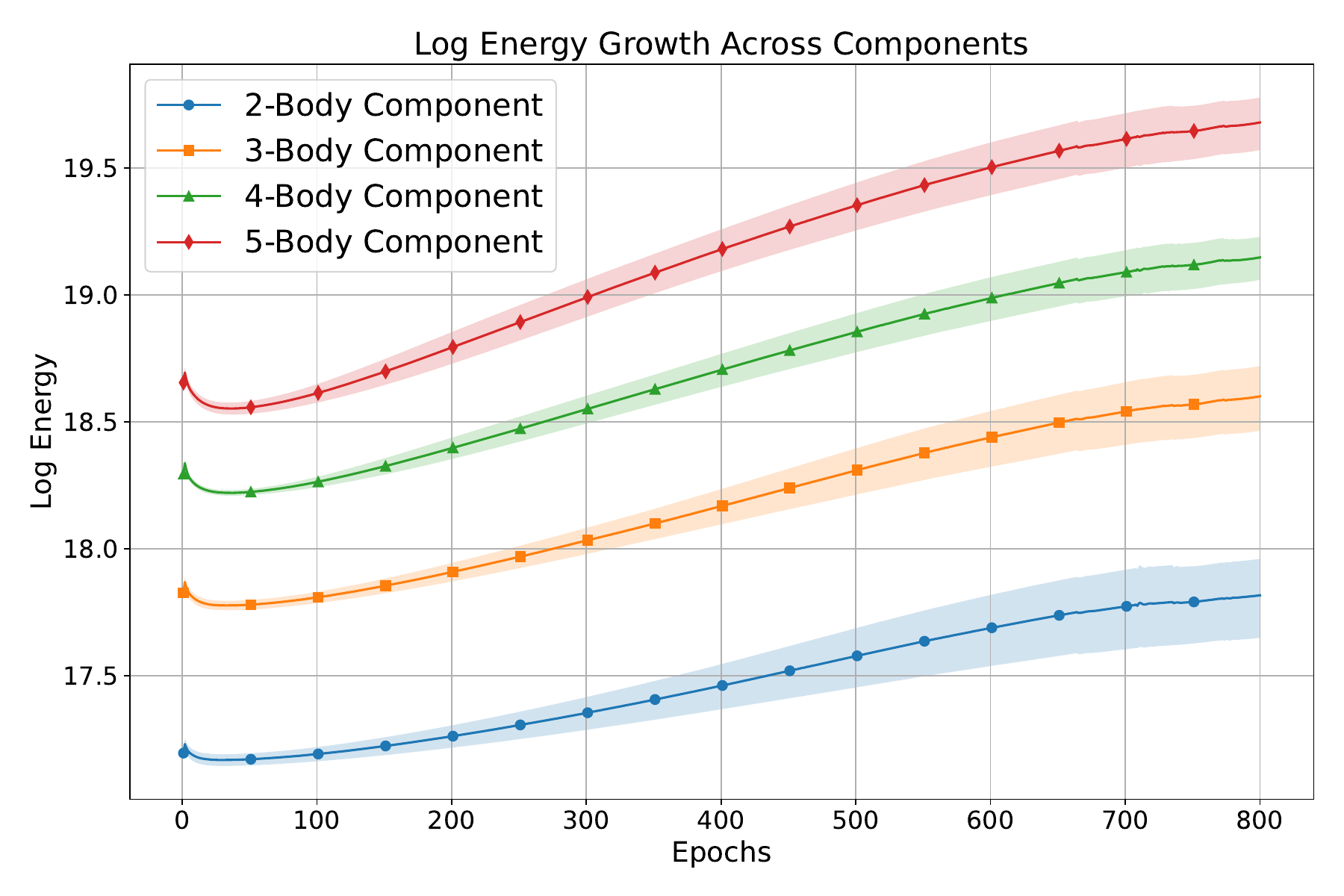}
    \caption{Log-energy growth from $k$-body term component ($k \le 5$) of Many-body MPNN.}
    \label{fig:indv_growth}
\end{figure}  

 


\subsection{Evaluating Mixing on Synthetic Spine Graphs}\label{app:sense}
To test the sensitivity bounds, we synthesized spine graphs with each node on the spine having multiple one-hop neighbor nodes attached only to it. If any edge on spine is erased, the graph becomes disconnected; hence these edges have generally low Ricci curvatures. 

To evaluate the mixing power of many-body interactions, we identify node groups of sizes [3 to $\nu$] and derive the Jacobians $\text{max}(\partial \bm{h_u^{(r)}} / \partial \bm{x}_v)$, with the maximum taken across embedding dimensions. We average the Jacobians from each group to represent the mixing power. We open-source the  \href{https://github.com/JThh/Many-Body-MPNN/blob/main/ManybodyMPNN_SyntheticZINC_OSQ_Playground.ipynb}{notebook} for future exploration.




\subsection{On Regressing Synthetic Graph Energies}


\subsubsection{Fixed-Shaped Graphs}\label{app:fix}
We also experiment on regular-shaped graphs such as \texttt{Ring}, \texttt{CrossedRing}, \texttt{CliquePath} as with \citet{topping2022understanding}, and \texttt{Spine} graphs having a long spine with each attached to a fixed number of one-hop neighbors. We synthesize graph energies in an alternative manner, based on their spectral properties: we take the absolute sum of eigenvalues of graph distance matrix that contains the pairwise shortest path lengths, or that of the graph adjacency matrix $A$, for emphasizing either node distances or local clustering.

The intuition for such energy construction is that graphs with more connectivity have wider spectral gaps, leading to larger target values, whilst making the formulation non-linear and not be modeled precisely with simple two-body message passing. 

Our many-body formulation models quite precisely the spectral properties of these regular-shaped graphs (Figure \ref{fig:layer}), and runs inductively and faster than eigen-decomposition on larger graphs due to lower time complexity (Proposition \ref{thm:time}). We observe that many-body MPNN often achieves the best performance when more layers (e.g., $\ge 32$) are stacked, indicating that it is less prone to over-squashing issues.
\begin{figure*}[ht]
    \centering
    \includegraphics[width=\linewidth]{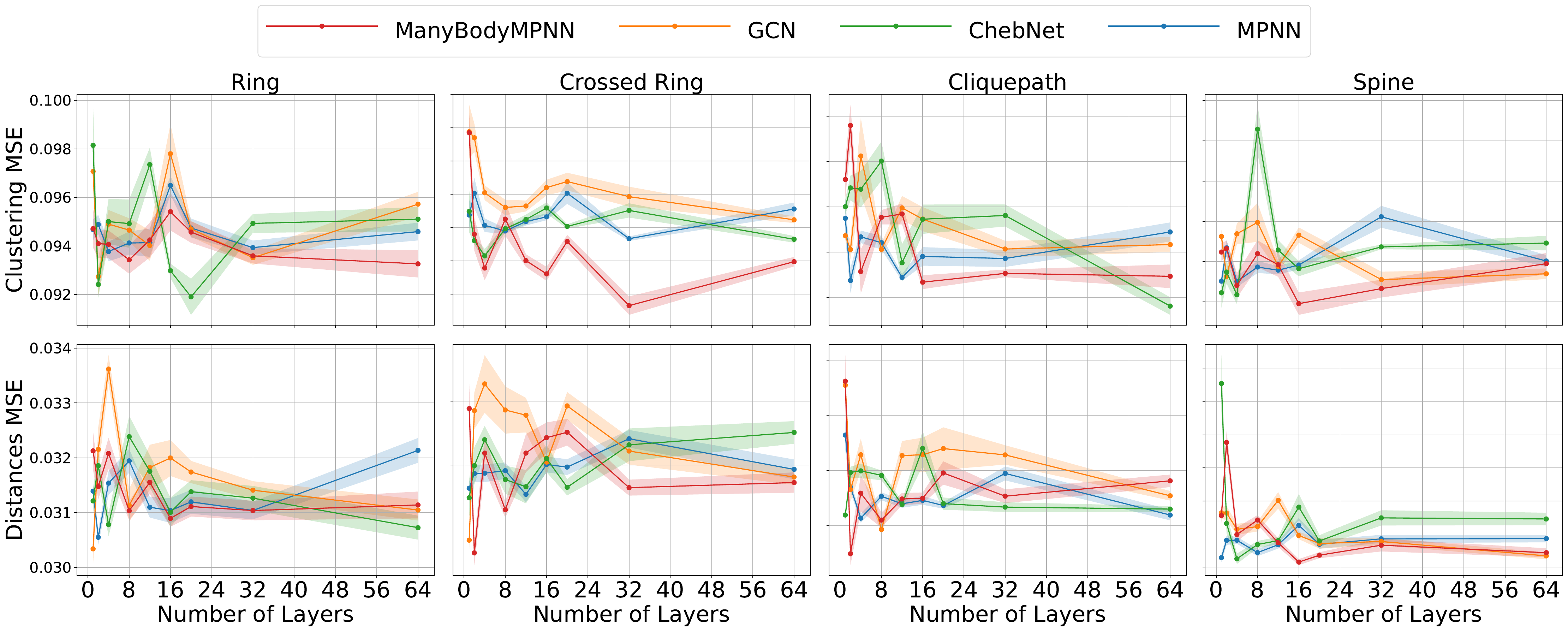}
    \caption{Test MSE on regressing the synthetic energies of regular-shaped graphs when varying the number of layers from 1 to 32 for different models. The graph energy functions are different between emphases on node distances and clustering and hence the two rows of plots only share the x-axis.}
    \label{fig:layer}
\end{figure*}




\subsubsection{On Correlation Orders}
We study different correlation orders and extract the patterns of performances on fixed-shaped graphs (see \cref{fig:maxorder}). We observe that generally across graph shapes, more layers lead to more fluctuations when increasing max orders, and more layers with high orders will likely lead to over-smoothing. We leave the detailed investigation of effective correlation orders with graph topology for future work.

\begin{figure*}[ht]
    \centering
    \includegraphics[width=\linewidth]{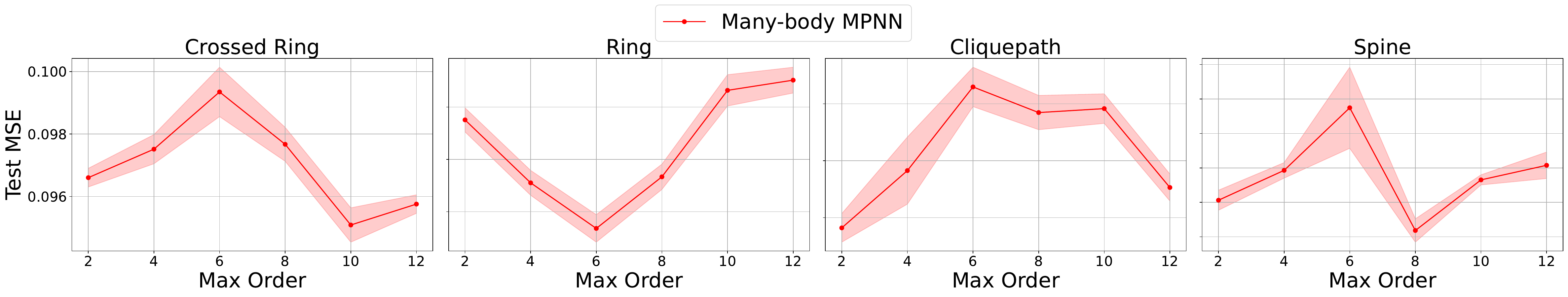}
    \caption{Test MSE for many-body MPNN of particular max correlation orders for regressing synthetic graph energies. There are 100 synthetic graphs of 300-500 nodes. The energy function is clustering-emphasized.}
    \label{fig:maxorder}
\end{figure*}







\section{Compute Resources}\label{app:compute}
We utilize one slice of A100 GPUs 80GB card (and one entire NVIDIA RTX 2080 Ti only for benchmarking) for speeding up graph convolutions and fully-connected layer operations, and AMD CPUs for graph synthesis and energy generation.

The experiments involved working with graph convolutional models on synthesized graphs of approximately 10,000 nodes with edge probabilities ranging from 0.15 to 0.3. The models consisted of 1-32 convolutional layers plus 2 linear layers, using the Adam optimizer, and with hidden dimensions ranging from 2 to 256. There are 5 types of graphs each with 100 instances, 4 types of models, and 10 repeated experiments for each setting. 

Based on the rough estimate of running 2 hours for each experiment series (e.g. \ref{app:fix}; there are 3 such series), we estimate $2\times 4 \times 3 \times 10 = 240$ GPU-hours to calculate the results in the paper, plus around 300 CPU-hours.







\end{document}